\documentclass[letterpaper]{article}
\usepackage[margin=1in]{geometry}

\usepackage{booktabs}       
\usepackage{amsmath}
\usepackage{amsthm}
\usepackage{amsfonts}       
\usepackage{nicefrac}       
\usepackage{xcolor}
\usepackage{soul}
\usepackage{dsfont}
\usepackage{hyperref}

\usepackage{graphicx}
\usepackage{times}

\newcommand{\A}{\mathcal{A}}
\newcommand{\R}{\mathbb{R}}
\newcommand{\E}{\mathbb{E}}
\newcommand{\U}{\mathcal{U}}
\newcommand{\Ftrue}{\mathcal{F}^\dagger}
\newcommand{\F}{\mathcal{F}}
\newcommand{\LL}{\mathcal{L}}

\newtheorem{theorem}{Theorem}
\newtheorem{lemma}[theorem]{Lemma}
\newtheorem{proposition}[theorem]{Proposition}

\title{
Neural Operator: Graph Kernel Network \\
for Partial Differential Equations
}

\author{Zongyi Li, Nikola Kovachki,  Kamyar Azizzadenesheli, Burigede Liu, \\
Kaushik Bhattacharya,  Andrew Stuart, Anima Anandkumar}
\date{\today}

\begin{document}

\maketitle

\begin{abstract}

The classical development of neural networks has been primarily for mappings between a finite-dimensional Euclidean space and a set of classes, or between two finite-dimensional Euclidean spaces.
The purpose of this work is to generalize neural networks so that they can learn mappings between infinite-dimensional spaces (operators). 
The key innovation in our work is that a single set of network parameters, within a carefully designed network architecture, may be used to describe mappings between infinite-dimensional spaces {\it and} between different finite-dimensional approximations of those spaces.
We formulate approximation of the infinite-dimensional mapping by composing nonlinear activation functions and a class
of integral operators. The kernel integration is computed by message passing on graph networks. This approach has substantial practical consequences which we will illustrate in the context of mappings between input data to partial differential equations (PDEs) and their solutions. 
In this context, such learned networks can generalize among different approximation methods for the PDE (such as finite difference or finite element methods) and among
approximations corresponding to different underlying levels of resolution and discretization.
Experiments confirm that the proposed graph kernel network does have the desired properties and show competitive performance compared to the state of the art solvers.  

\end{abstract}

\section{INTRODUCTION}
There are numerous applications in which it is desirable to
learn a mapping between Banach spaces. In particular, either
the input or the output space, or both, may be infinite-dimensional.
The possibility of learning such mappings opens up a new class of
problems in the design of neural networks, with widespread potential
applicability. New ideas are required to build on traditional 
neural networks which are mappings from finite-dimensional Euclidean spaces
into classes, or into another finite-dimensional Euclidean space.
We study the development of neural networks in the setting in which 
the input and output spaces comprise real-valued functions defined 
on subsets in $\mathbb{R}^d.$

\subsection{Literature Review And Context}

We formulate a new class of neural networks, which are defined
to map between spaces of functions on a bounded open set $D$ in $\mathbb{R}^d.$  
Such neural networks, once trained, have the important property that they
are discretization invariant, sharing the same network parameters
between different discretizations. In contrast, standard neural network architectures depend heavily on the discretization and have difficulty
in generalizing between different grid representations. Our methodology
has an underlying Nystr\"om approximation formulation \cite{nystrom1930praktische} 
which links
different grids to a single set of network parameters. We illustrate
the new conceptual class of neural networks within the context of
partial differential equations, and the mapping between input data
(in the form of a function) and output data (the function which solves
the PDE). Both supervised and semisupervised settings are considered.

In PDE applications, the defining equations are often local, whilst the solution operator has non-local effects which, nonetheless, decay. Such non-local effects can be
described by integral operators with graph approximations of Nystr\"om type \cite{belongie2002spectral}
providing a consistent way of connecting different grid or data structures arising in
computational methods. For this reason, graph networks hold great potential for the solution
operators of PDEs, which is the point of departure for our work.

\paragraph{Partial differential equations (PDEs).} A wide range of important engineering and physical problems are governed by PDEs. Over the past few decades, significant progress has been made on formulating \cite{gurtin1982introduction} and solving \cite{johnson2012numerical} the governing PDEs in many scientific fields from micro-scale problems (e.g., quantum and molecular dynamics) to macro-scale applications (e.g., civil and marine engineering). Despite the success in the application of PDEs to solve real-life problems,  two significant challenges remain.  First, identifying/formulating the underlying PDEs appropriate for the modeling of a specific problem usually requires extensive prior knowledge in the corresponding field which is then combined with universal
conservation laws to design a predictive model; for example, modeling the deformation and fracture of solid structures requires detailed knowledge on the relationship between stress and strain in the constituent material. For complicated systems such as living cells, acquiring such knowledge is often elusive and formulating the governing PDE for these systems remains prohibitive; the possibility of
learning such knowledge from data may revolutionize such fields. Second, solving complicated non-linear PDE systems (such as those arising in turbulence and plasticity) is computationally demanding; again the
possibility of using instances of data from such computations to design fast approximate solvers
holds great potential. In both these challenges, if neural networks are to play a role in exploiting
the increasing volume of available data, then there is
a need to formulate them so that they are well-adapted to mappings between function spaces.

We first outline two major neural network based approaches for PDEs.
We consider PDEs of the form
 \begin{align}
\label{eq:generalpde0}
\begin{split}
(\LL_a u)(x) &= f(x), \qquad x \in D\ \\
u(x) &= 0, \qquad \quad \:\: x \in \partial D,
\end{split}
\end{align}
with solution $u: D \to \mathbb{R}$, and parameter $a:D \to \mathbb{R}$ entering
the definition of $\LL_a$. The domain $D$ is discretized into $K$ points (see Section \ref{sec:problemsetting})
and $N$ training pairs of coefficient functions and (approximate) solution functions \(\{a_j, u_j\}_{j=1}^N\) are used to design a neural network.
The first approach parametrizes the solution operator as a deep convolutional neural network between finite Euclidean space \(\F : \R^K \times \Theta \to \R^K\) \cite{guo2016convolutional, Zabaras, Adler2017, bhatnagar2019prediction}. Such an approach is, by definition, not mesh independent and will need modifications to the architecture for different resolution and discretization of \(K\) in order to achieve consistent error (if at all possible). We demonstrate this issue in section \ref{sec:numerics} using the architecture of \cite{Zabaras} which was designed for the solution of \eqref{eq:ellptic} on a uniform \(64 \times 64\) mesh. Furthermore, this approach is limited to the discretization size and geometry of the training data hence it is not possible to query solutions at new points in the domain. In contrast we show, for our method, both invariance of the error to grid resolution, and the ability to transfer the solution between meshes in section \ref{sec:numerics}. 

The second approach directly parameterizes the solution \(u\) as a neural network \(\F : D \times \Theta \to \R\) 
\cite{Weinan, raissi2019physics,bar2019unsupervised}. This approach is, of course, mesh independent since the solution is defined on the physical domain.
However, the parametric dependence is accounted for in a mesh-dependent fashion. 
Indeed, for any given new equation with new coefficient function \(a \), one would need to train a new neural network \(\F_a\). Such an approach closely resembles classical methods such as finite elements, replacing the linear span of a finite set of local basis functions with the space of neural networks. This approach suffers from the same computational issue as the classical methods: one needs to solve an optimization problem for every new parameter. Furthermore, the approach is limited to a setting in which the underlying PDE is known; purely data-driven learning of a map between spaces of functions is not possible. The methodology we introduce circumvents these issues.

Our methodology most closely
resembles the classical reduced basis method \cite{DeVoreReducedBasis} or the method of \cite{cohendevore}. 
Along with the contemporaneous work \cite{Kovachki}, our method, to the best of our knowledge, 
is the first practical deep learning method that is designed to learn maps between infinite-dimensional spaces. It
remedies the mesh-dependent nature of the approach in \cite{guo2016convolutional, Zabaras, Adler2017, bhatnagar2019prediction} by producing a quality of approximation that is invariant to the resolution of the function and it has the ability to transfer solutions between meshes. Moreover it need only be trained once on the equations set  \(\{a_j, u_j\}_{j=1}^N\); then, obtaining a solution for a new \(a \sim \mu\), only requires a forward pass of the network, alleviating the major computational issues incurred in \cite{Weinan, raissi2019physics, herrmann2020deep, bar2019unsupervised}. Lastly, 
our method requires no knowledge of the underlying PDE; the true map \(\Ftrue\) can be treated as a black-box, perhaps trained on experimental data or on the output of a costly computer simulation, not necessarily a PDE.

\paragraph{Graph neural networks.} Graph neural network (GNNs), a class of neural networks that apply on graph-structured data, have recently been developed and seen a variety of applications. Graph networks incorporate an array of techniques such as graph convolution, edge convolution, attention, and graph pooling  \cite{kipf2016semi,hamilton2017inductive,gilmer2017neural,velivckovic2017graph,murphy2018janossy}. 
GNNs have also been applied to the modeling of physical phenomena such as molecules \cite{chen2019graph} and rigid body systems \cite{battaglia2018relational}, as these problems exhibit a natural graph interpretation: the particles are the nodes and the interactions are the edges. 

The work \cite{pmlr-v97-alet19a} performed an initial study that employs graph networks on the problem of learning solutions to Poisson's equation among other physical applications. They propose an encoder-decoder setting, constructing graphs in the latent space and utilizing message passing between the encoder and decoder. However, their model uses a nearest neighbor structure that is unable to capture non-local dependencies as the mesh size is increased.
In contrast, we directly construct a graph in which the nodes are located on the spatial domain of the output function. Through message passing, we are then able to directly learn the kernel of the network
which approximates the PDE solution. When querying a new location, we simply add a new node to our spatial graph and connect it to the existing nodes, avoiding interpolation error by leveraging the power of the Nystr\"om extension for integral operators.

\begin{figure}[t]
    \centering
    \includegraphics[width=12cm]{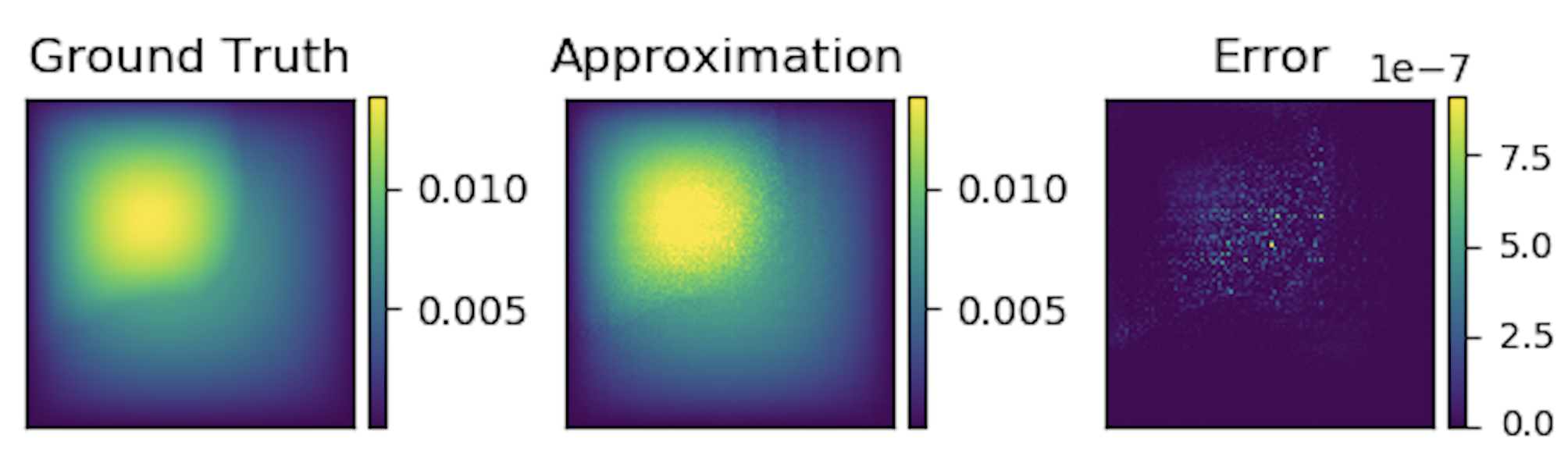}\\
    \label{fig:generalization}
    \small{
    Graph kernel network for the solution of \eqref{eq:ellptic}. It can be trained on a small resolution and will generalize to a large one. Error is point-wise squared error.}
    \caption{Train on $16 \times 16$, test on $241 \times 241$}
\end{figure}

\paragraph{Continuous neural networks.}
The concept of defining neural networks in infinite-dimensional spaces is a central problem that long been studied \cite{Williams, Neal, BengioLeRoux,GlobersonLivni, Guss}. The general idea is to take the infinite-width limit which yields a non-parametric method and has connections to Gaussian Process Regression \cite{Neal,MathewsGP,Garriga-AlonsoGP}, leading to the introduction of
deep Gaussian processes \cite{damianou2013deep,aretha}. Thus far, such methods have not yielded efficient numerical algorithms that can parallel the success of convolutional or recurrent neural networks in finite dimensions in the setting of mappings between function spaces.  Another idea is to simply define a sequence of compositions where each layer is a map between infinite-dimensional spaces with a finite-dimensional parametric dependence. This is the approach we take in this work, going a step further by sharing parameters between each layer.

\subsection{Contributions}


We introduce the concept of Neural Operator and instantiate it through \emph{graph kernel networks}, a novel deep neural network method to learn the mapping between infinite-dimensional spaces of functions defined on bounded open subsets of $\mathbb{R}^d$.

\begin{itemize}
\item Unlike existing methods, our approach is demonstrably able to learn the mapping between function spaces, and is invariant to different approximations and grids,  as demonstrated in Figure \ref{fig:generalization}.

\item We deploy the Nystr\"om extension to connect the neural network on function space to families
of GNNs on arbitrary, possibly unstructured, grids.

\item We demonstrate that the Neural Operator approach has competitive approximation accuracy to classical and deep learning methods. 
    
\item Numerical results also show that the Neural Operator needs only to be trained on a few samples in order to generalize to the whole class of problems. 


\item We show the ability to perform semi-supervised learning, learning from data at only a few points and then generalizing to the whole domain.
\end{itemize}

These concepts are illustrated in the context of a family of elliptic PDEs prototypical
of a number of problems arising throughout the sciences and engineering.



    



\section{PROBLEM SETTING}
\label{sec:problemsetting}
Our goal is to learn a mapping between two infinite dimensional spaces by using a finite
collection of observations of input-output pairs from this mapping: supervised learning. Let \(\A\) and \(\U\) 
be separable Banach spaces and \(\Ftrue : \A \to \U\)
a (typically) non-linear map. Suppose we have observations \(\{a_j, u_j\}_{j=1}^N\) where 
\(a_j \sim \mu\) is an i.i.d. sequence from the probability measure \(\mu\) supported on 
\(\A\) and \(u_j = \Ftrue(a_j)\) is possibly corrupted with noise. We aim to build an approximation of \(\Ftrue\) by 
constructing a parametric map 
\begin{equation}
\label{eq:approxmap}
\F : \A \times \Theta \to \U
\end{equation}
for some finite-dimensional parameter space \(\Theta\) and then choosing
\(\theta^\dagger \in \Theta\) so that \(\F(\cdot, \theta^\dagger) \approx \Ftrue\).

This is a natural framework for learning in infinite-dimensions as one could define a cost functional \(C : \U \times \U \to \R\) and seek a minimizer of the problem
\[\min_{\theta \in \Theta} \E_{a \sim \mu} [C(\F(a,\theta), \Ftrue(a))]\]
which directly parallels the classical finite-dimensional 
setting \cite{Vapnik1998}. Showing the existence of minimizers, in the infinite-dimensional setting, remains a challenging open problem. We will approach this problem in the test-train setting in which
empirical approximations to the cost are used. We conceptualize our methodology in the
infinite-dimensional setting. This means that all finite-dimensional approximations
can share a common set of network parameters which are defined in the (approximation-free) infinite-dimensional setting. To be concrete we will consider infinite-dimensional spaces which are Banach
spaces of real-valued functions defined on a bounded open set in $\mathbb{R}^d$. We then consider
mappings $\Ftrue$ which take input functions to a PDE and map them to solutions of the PDE,
both input and solutions being real-valued functions on $\mathbb{R}^d$.

A common instantiation of the preceding problem is the approximation of 
the second order elliptic PDE
\begin{align}
\label{eq:ellptic}
\begin{split}
- \nabla \cdot (a(x) \nabla u(x))  &= f(x), \quad  x \in D \\
u(x) &= 0, \qquad \:\: x \in \partial D
\end{split}
\end{align}
for some bounded, open set \(D \subset \R^d\) and a fixed function
{\(f \in L^2(D;\R)\)}. This equation is prototypical of PDEs arising in
numerous applications including hydrology \cite{bear2012fundamentals} and elasticity \cite{antman2005problems}. 
For a given  \(a \in \A = L^\infty(D;\R^+) \cap L^2(D;\R^+)\), equation (\ref{eq:ellptic}) has 
a unique weak solution \(u \in \U = H_0^1(D;\R)\) \cite{Evans} 
and therefore we can define the solution operator \(\Ftrue\)
as the map \(a \mapsto u\). Note that while the PDE (\ref{eq:ellptic})
is linear, the solution operator  $\Ftrue$ is not.

Since our data \(a_j\) and \(u_j\) are , in general, functions, to work with them 
numerically, we assume access only to point-wise evaluations. To illustrate this, we will
continue with the example of the preceding paragraph. To this end 
let \(P_K = \{x_1,\dots,x_K\} \subset D\) be a \(K\)-point discretization of the domain
\(D\) and assume we have observations \(a_j|_{P_K}, u_j|_{P_K} \in \R^K\), for a finite
collection  of input-output pairs indexed by $j$.
In the next section, we propose a kernel inspired graph neural network 
architecture which, while trained on the discretized data, can produce an
answer \(u(x)\) for any \(x \in D\) given a new input \(a \sim \mu\).
That is to say that our approach is independent of the discretization
\(P_K\) and therefore a true function space method; we verify this claim numerically by showing invariance of the error as \(K \to \infty\). 
Such a property is highly desirable as it allows a transfer of solutions between different grid geometries and discretization sizes.

We note that, while the application of our methodology is based on having point-wise evaluations of the function, it is not limited by it. One may, 
for example, represent a function numerically as a finite set of truncated basis coefficients. Invariance of the representation would then be with respect to the size of this set. Our methodology can, in principle, be modified to accommodate this scenario through a suitably chosen architecture. We do not pursue this direction in the current work.

\section{GRAPH KERNEL NETWORK}
\label{sec:method}
We propose a graph kernel neural network for the solution of the problem outlined in 
section \ref{sec:problemsetting}. A {\bf table of notations} is included in Appendix \ref{table:notation}. As a guiding principle for our architecture, we take 
the following example. Let \(\LL_a\) be a differential operator depending on a 
parameter \(a \in \A\) and consider the PDE
\begin{align}
\label{eq:generalpde}
\begin{split}
(\LL_a u)(x) &= f(x), \qquad x \in D\ \\
u(x) &= 0, \qquad \quad \:\: x \in \partial D
\end{split}
\end{align}
for a bounded, open set \(D \subset \R^d\) and some fixed function \(f\) living in an appropriate function space determined by the structure of \(\LL_a\). The elliptic operator {\(\LL_a \cdot = - \text{div} (a \nabla \cdot)\)} from equation (\ref{eq:ellptic}) is an example. Under fairly general 
conditions on \(\LL_a\) \cite{Evans}, we may define the Green's function \(G : D \times D \to \R\) as the 
unique solution to the problem 
\[\LL_a G(x, \cdot) = \delta_x\]
where \(\delta_x\) is the delta measure on \(\R^d\) centered at \(x\). Note that \(G\) will depend on the parameter \(a\) thus we will henceforth denote it as \(G_a\). The solution to \eqref{eq:generalpde} can then be represented as
\begin{equation}
\label{eq:generalsolution}
u(x) = \int_D G_a(x,y)f(y) \: dy.
\end{equation}

This is easily seen via the formal computation
\begin{align*}
    (\LL_a u)(x) &= \int_D (\LL_a G(x,\cdot))(y) f(y) \: dy \\
    &= \int_D \delta_x(y) f(y) \: dy \\
    &= f(x)
\end{align*}

Generally the Green's function is continuous at points \(x \neq y\), for example, when
\(\LL_a\) is uniformly elliptic \cite{gilbarg2015elliptic}, hence it is natural to model it via a neural network.  
Guided by the representation \eqref{eq:generalsolution}, we propose the following iterative 
architecture for \(t = 0,\dots,T-1\).  
\begin{align}\label{eq:kernel}
\begin{split}
v_{t+1}(x) = \sigma \bigg ( W v_t(x) 
\!\!+\!\! \int_D \kappa_\phi (x, y, a(x), a(y)) v_t(y) \: \nu_x(dy) \bigg )\!\!\!\! 
\end{split}
\end{align}
where \(\sigma : \R \to \R\) is a fixed function applied element-wise, \(\nu_x\) is a fixed Borel measure for each \(x \in D\). The matrix \(W \in \R^{n \times n}\), together with the parameters $\phi$ 
entering kernel \(\kappa_\phi : \R^{2(d+1)} \to \R^{n \times n}\), are to be learned from data.
We model $\kappa_\phi$ as a neural network mapping $\mathbb{R}^{2(d+1)}$ to $\R^{n \times n}.$

Discretization of the continuum picture may be viewed as replacing Borel measure $\nu_x$ by an
empirical approximation based on the $K$ grid points being used.
In this setting we may view \(\kappa_\phi \) as a \(K\times K\) kernel block matrix, where each entry \(\kappa_\phi(x,y)\) is itself a \(n \times n\) matrix. Each block shares the same set of network
parameters. This is the key to making a method that shares common parameters independent
of the discretization used.

Finally, we observe that, although we have focussed on neural networks mapping $a$ to $u$, generalizations are possible, such as mapping $f$ to $u$, or having non-zero boundary data $g$
on $\partial D$ and mapping $g$ to $u$. More generally one can consider the mapping from
$(a,f,g)$ into $u$ and use similar ideas. Indeed to illustrate ideas we will consider the mapping
from $f$ to $u$ below (which is linear and for which an analytic solution is known)
before moving on to study the (nonlinear) mapping from $a$ to $u$.

\paragraph{Example 1: Poisson equation.} We consider a simplification of the
foregoing in which we study the map from $f$ to $u$. To this end we set \(v_0(x) = f(x)\), \(T=1\), \(n=1\),
\(\sigma(x) = x\), \(W = w = 0\), and \(\nu_x(dy) = dy\) (the Lebesgue measure) in \eqref{eq:kernel}.
We then obtain the representation \eqref{eq:generalsolution} with the Green's 
function \(G_a\) parameterized by the 
neural network \(\kappa_\phi\) with explicit dependence on \(a(x)\), \(a(y)\).
Now consider the setting where \(D = [0,1]\)
and \(a(x) \equiv 1 \), so that \eqref{eq:ellptic} reduces to the 1-dimensional Poisson equation 
with explicitly computable Green's function. Indeed,
\[G(x,y) = \frac{1}{2} \left ( x + y - |y-x| \right ) - xy. \]
Note that although the map \(f \mapsto u\) is, in
function space, linear,  the Green's function  itself is not linear in either argument. 
Figure \ref{fig:kernel1d} shows \(\kappa_\phi\) after training with $N=2048$ samples \(f_j \sim \mu = \mathcal{N}(0, (-\Delta + I)^{-1})\) with 
periodic boundary conditions on the operator $-\Delta+I$.
(samples from this measure
can be easily implemented by means of a random Fourier series -- Karhunen-Loeve -- see
\cite{Lord}). 

\begin{figure}[t]
    \centering
    \includegraphics[width=6cm]{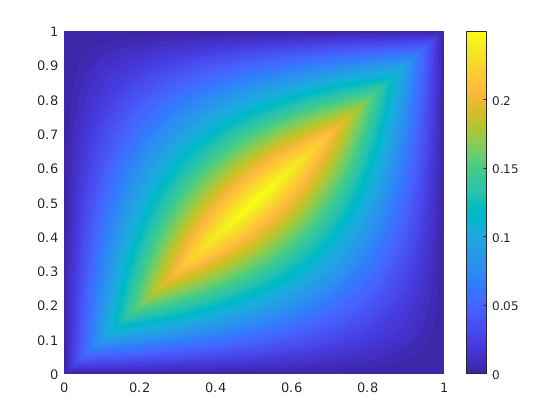}
    \includegraphics[width=6cm]{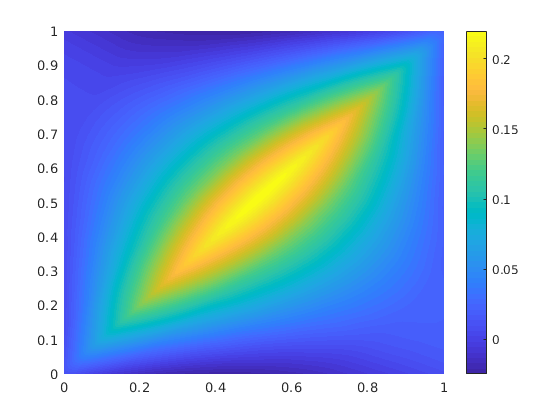}\\
\small{Proof of concept: graph kernel network on $1$ dimensional Poisson equation; comparison
of learned and truth kernel.}
\caption{Kernel for one-dimensional Green's function}
\label{fig:kernel1d}
\end{figure}

\paragraph{Example 2: 2D Poisson equation.} We further demonstrate the power of the graph kernel network by extending the Poisson equation studied in example 1 to the two dimensional (2D) case, where we approximate the map \(f(x) \mapsto u(x)\) where \(x \in D = [0,1] \times [0,1] \). We consider two approaches: $(i)$ graph kernel network to approximate the 2D Green's function \(G(x,y)\) and $(ii)$ dense neural network with \(f(x)\) as input and \(u(x)\) as output such that the mapping \(f(x) \mapsto u(x)\) is directly approximated.

The two neural networks are trained with the same training sets of different sizes ranging from 1 to 100 samples and tested with 1000 test samples. The relative \(l_2\) test errors as a function of training samples are illustrated in Figure \ref{fig:compareGKN}.
We observe that the Graph kernel network approximates the map with a minimum \( (\approx 5)\)\ 
number of training samples while still possessing a smaller test error comparing to that of a dense neural network trained with 100 samples. Therefore, the Graph kernel network potentially significantly reduces the number of required training samples to approximate the mapping. This property is especially important in practice as obtaining a huge number of training data for certain engineering/physics problems is always prohibitive. 

The reason for the Graph kernel network to have such strong approximation power is because that it is able to learn the truth Green's function for the Poisson equation, as already demonstrated in the 1D case. We shall demonstrate further in Appendix A.4 that the Graph kernel networks are apt to capture the truth kernel for the 2D Poisson equation.  

\begin{figure}[t]
    \centering
    \includegraphics[width=12cm]{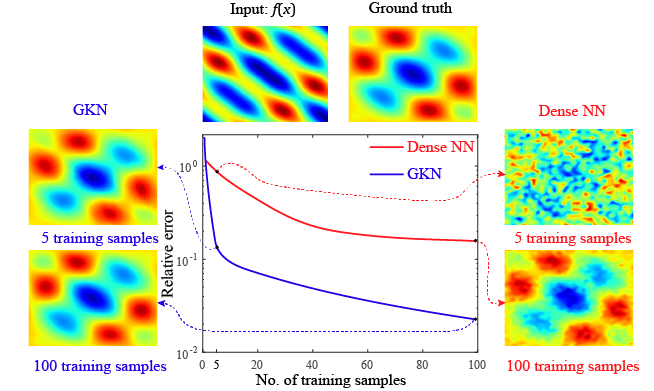}\\
    \label{fig:compareGKN}
    \caption{Comparison between the performance of graph kernel network and the dense neural network on the approximation of \eqref{eq:ellptic}. Plots of relative \(l_2\) test approximation errors versus the number of training samples for the Graph kernel network and dense neural networks when approximating the problem \eqref{eq:ellptic}. 
    }
\end{figure}

\paragraph{Algorithmic framework.} 
The initialization \(v_0(x)\) to our network \eqref{eq:kernel} can be viewed as the initial guess we make for the solution \(u(x)\) as well as any other dependence we want to make explicit. A natural choice is to start with the coefficient \(a(x)\) itself as well as the position in physical space \(x\). This $(d+1)$-dimensional vector field is then lifted to a $n$-dimensional vector field, an operation which we may view
as the first layer of the overarching neural network. This is then used as an initialization to the kernel neural network, which is iterated $T$ times. In the final layer,
we project back to the scalar field of interest with another neural network layer.

Due to the smoothing effect of the inverse elliptic operator in \eqref{eq:ellptic} with respect
to the input data $a$ (and indeed $f$ when we consider this as input), we augment the initialization $(x, a(x))$ with a Gaussian smoothed version of the coefficients  $a_{\epsilon}(x)$, together with 
their gradient $\nabla a_{\epsilon}(x)$. Thus we initialize with a $2(d+1)$-dimensional vector field. Throughout this paper the Gaussian
smoothing is performed with a centred isotropic Gaussian with variance $5.$
The Borel measure $\nu_x$ is chosen to be the Lebesgue measure
supported on a ball at $x$ of radius $r$. Thus we have
\begin{align}
&v_0(x) = P(x, a(x), a_{\epsilon}(x), \nabla a_{\epsilon}(x)) + p\\
\begin{split}\label{eq:int}
v_{t+1}(x) = \sigma\Big( W v_t(x)  
+ \int_{B(x,r)} \!\!\!\kappa_{\phi}\big(x,y,a(x),a(y)\big)
v_t(y)\: \mathrm{d}y \Big)
\end{split}\\
&u(x) = Q  v_T (x) + q
\end{align}
where $P \in \R^{n \times 2(d+1)}$, $p \in \R^{n}$, $v_{t}(x) \in \R^n$ and $Q \in \R^{1 \times n}$, $q \in \R$.
The integration in (\ref{eq:int}) is approximated by a Monte Carlo sum via a message passing graph network with edge weights $(x,y,a(x),a(y))$. The choice of measure \(\nu_x(\mathrm{d}y) = \mathds{1}_{B(x,r)} \mathrm{d}y\) is two-fold: 1) it allows for more efficient computation and 2) it exploits the decay property of the Green's function. Note that if more information is known about the true kernel, it can be added into this measure. For example, if we know the true kernel has a Gaussian 
structure, we can define \(\nu_x(\mathrm{d}y) = \mathds{1}_{B(x,r)} \rho_x(y) \mathrm{d}y\) where \(\rho_x(y)\) is a Gaussian density. Then \(\kappa_\phi\) will need to learn a much less complicated function. We however do not pursue this direction in the current line of work.

\paragraph{Message passing graph networks.}
Message passing graph networks comprise a standard architecture employing edge features~\cite{gilmer2017neural}. If we properly construct the graph on the spatial domain $D$ of the PDE, the kernel integration can be viewed as an aggregation of messages.
Given node features $v_t(x) \in \mathbb{R}^{n}$, edge features $e(x,y) \in \mathbb{R}^{n_e}$, and a graph $G$, the message passing neural network with averaging aggregation is
\begin{equation}\label{eq:mpnn}
v_{t+1}(x) =  \sigma\Bigl(W v_t(x) + 
\frac{1}{|N(x)|} \!\!\sum_{y \in N(x)} \!\!\!\kappa_{\phi}\big(e(x,y)\big) v_t(y)\Bigr)\!\!
\end{equation}
where $W \!\!\in\! \mathbb{R}^{n \times n}$, $N(x)$ is the neighborhood of $x$ according to the graph, $\kappa_{\phi}\big(e(x,y)\big)$ is a neural network taking as input edge features and as output 
a matrix in $\mathbb{R}^{n \times n}$. Relating to \eqref{eq:int}, $e(x,y) \!=\! (x,y,a(x),a(y)) \!\in \!\R^{2(d+1)}$.

\paragraph{Graph construction.} 
To use the message passing framework (\ref{eq:mpnn}), we need to design a graph that connects the physical domain \(D\) of the PDE. The nodes are chosen to be the \(K\) discretized spatial locations. Here we work on a standard uniform mesh, but there are many other possibilities such as finite-element triangulations and random points at which data is acquired. The edge connectivity is then chosen according to the integration measure in \eqref{eq:int}, namely Lebesgue restricted to a ball. Each node \(x \in \R^d\) is connected to nodes which lie within \(B(x,r)\), defining the neighborhood set \(N(x)\). Then for each neighbor \(y \in N(x)\),
we assign the edge weight \(e(x,y) = (x,y,a(x),a(y))\). Equation \eqref{eq:mpnn} can then be viewed as 
a Monte Carlo approximation of \eqref{eq:int}. This local structure allows for more efficient computation while remaining invariant to mesh-refinement. Indeed, since the radius parameter \(r\) is chosen in physical space, the size of the set \(N(x)\) grows as the discretization size \(K\) grows. This is a 
key feature that makes our methodology mesh-independent.

\paragraph{Nystr\"om approximation of the kernel.}
While the aforementioned graph structure severely reduces the computational overhead of integrating over the entire domain \(D\) (corresponding to a fully-connected graph), 
the number of edges still scales like \(\mathcal{O}(K^2)\). To overcome this, we employ a random Nystr\"om-type approximation of the kernel. In particular, we uniformly sample 
\(m \ll K\) nodes from the original graph, constructing a new random sub-graph. This process is repeated \(l \in \mathbb{N}\) times, yielding \(l\) random sub-graphs each with \(m\) nodes. This can be thought of as a way of reducing the variance in the estimator. We use these sub-graphs
when evaluating \eqref{eq:mpnn} during training, leading to the more favorable scaling \(\mathcal{O}(lm^2)\). Indeed, numerically we find that \(l=4\) and \(m=200\) is sufficient even when \(K = 421^2 = 177,241\). In the evaluation phase, when we want the solution on a particular mesh geometry, we simply partition the mesh into sub-graphs each with \(m\) nodes and evaluate each separately. 

We will now highlight  the quality of this kernel approximation in a RHKS setting. A real Reproducing Kernel Hilbert Space (RKHS)  $(\mathcal{H}, \langle \cdot, \cdot \rangle, \|\cdot\|)$ is a Hilbert space of functions $f:D\to \R$ where point-wise evaluation is a continuous linear functional, i.e. $|f(x)| \leq C \|f\|$ for some constant $C \geq 0$, independent of $x$. For every RHKS, there exists a unique, symmetric, positive definite kernel $\kappa: D \times D \to \R$, which gives the representation  $f(x) = \langle f, \kappa(\cdot,x) \rangle$.
Let $T : \mathcal{H} \to \mathcal{H}$ be a linear operator on \(\mathcal{H}\) acting via the kernel
\begin{equation*}
    Tf = \int_{B(\cdot,r)} \kappa(\cdot,y)f(y)\mathrm{\nu}(dy).
\end{equation*}
Let \(T_m : \mathcal{H} \to \mathcal{H}\) be its \(m\)-point empirical approximation
\begin{align*}
    &T_m = \int_{B(\cdot,r)} \kappa(\cdot,y)f(y)\mathrm{\nu}_m(dy),\\
    &\nu_m(dy)=\frac{1}{m}\sum_{k=1}^m \delta_{y_k}(dy),
\end{align*}
so that
\begin{equation*}
    T_m f = \frac{1}{m} \sum_{k=1}^m \kappa(\cdot,y_k)f(y_k).
\end{equation*}
The error of this approximation achieves the Monte Carlo rate $O(m^{-1/2})$:
\begin{proposition}
\label{prop:montecarlo}
Suppose $\E_{y\sim\nu}[\kappa(\cdot,y)^4] < \infty$ then there exists a constant $C \geq 0$
such that
\[ \E\|T - T_m\|_{HS} \leq \frac{C}{\sqrt{m}}\]
where \(\|\cdot\|_{HS}\) denotes the Hilbert-Schmidt norm on operators acting on \(\mathcal{H}\).
\end{proposition}
For proof see Appendix \ref{app:proof}. With stricter assumptions similar results can also be proven with high probability \cite{Belkin}.

\begin{lemma}\label{lemma:Belkin}(cite Belkin)
$T$ and $T_K$ are Hilbert-Schmidt. Furthermore, with probability greater than $1-2e^{\tau}$
\[\|T - T_K\|_{HS} \leq \frac{2\sqrt{2}k\sqrt{\tau}}{\sqrt{K}}\]
where $k = \sup_{x\in D}{\kappa(x,x)}.$
\end{lemma}

We note that, in our algorithm, $\kappa: D \times D \to \R^{n \times n }$ whereas the preceding results are proven only in the setting $n=1$; nonetheless they provide useful intuition regarding the approximations used in our methodology.

\section{EXPERIMENTS}
\label{sec:numerics}

In this section we illustrate the claimed properties of our methodology, 
and compare it to existing approaches in the literature. All experimental results 
concern the mapping \(a \mapsto u\) defined by \eqref{eq:ellptic} with \(D = [0,1]^2\). Coefficients 
are generated according to \(a \sim \mu \) where \(\mu = \psi_{\#} \mathcal{N}(0,(-\Delta + 9I)^{-2})\)
with a Neumann boundry condition on the operator \(-\Delta + 9I\).
The mapping \(\psi : \R \to \R\) takes the value 12 on the positive part of the real line and 3 on the negative. Hence the coefficients are piecewise constant with random geometry and a 
fixed contrast of 4. Such constructions are prototypical of physical properties such as
permeability in sub-surface flows and material microstructures in elasticity. Solutions \(u\)
are obtained by using a second-order finite difference scheme on a $241 \times 241$ grid.
Different resolutions are downsampled from this dataset.

To be concrete we set the dimension of representation $n$ (i.e. the width of graph network) to be $64$, the number of iterations $T$ to be $6$, \(\sigma\) to be the ReLU function, and the inner kernel network $\kappa$ to be a $3-$layer feed-forward network with widths $(6, 512, 1024, n^2)$ and ReLU activation. We use the Adam optimizer with a learning rate $1e-4$ and train for $200$ epochs with respect to the absolute mean squared error on the normalized data unless otherwise stated. These chosen hyperparameters are not optimized and could be adapted to improve performance. 
We employ the message passing network from the standard Pytorch graph network library Torch-geometric \cite{Fey/Lenssen/2019}. All the test errors are relative $L^2(D)$ errors on the original data. A {\bf table of notations} for the hyper-parameters can be found in Appendix \ref{table:notation}. The {\bf code} and {\bf data} can be found at \url{https://github.com/wumming/graph-pde}.

\subsection{Supervised Setting}
First we consider the supervised scenario that we are given $N$ training pairs $\{a_j, u_j\}_1^N$, where each $a_j$ and $u_j$ are provided on a $s \times s$ grid (\(K = s^2\)).

\paragraph{Generalization of resolutions on full grids.}
To examine the generalization property, we train the graph kernel network on resolution $s \times s $ and test on another resolution $s'\times s'$. We fix the radius to be $r=0.10$, train on $N=100$ equation pairs and test on $40$ equation pairs.

\begin{table}[h]
\caption{Comparing resolutions on full grids}
\label{table:1}
\begin{center}
\begin{tabular}{l|lll}
\multicolumn{1}{c}{\bf Resolutions}  &\multicolumn{1}{c}{\bf $s'=16$}
&\multicolumn{1}{c}{\bf $s'=31$}
&\multicolumn{1}{c}{\bf $s'=61$} \\
\hline 
$s=16$  &$0.0525$ &$0.0591$ &$0.0585$\\
$s=31$  &$0.0787$  &$0.0538$ &$0.0588$ \\
\hline
\end{tabular}\\
\small{$r=0.10$, $N=100$, relative $l_2$ test error}
\end{center}

\end{table}

As shown in Table \ref{table:1}, for each row, the test errors at different resolutions remain on the same scale, illustrating the desired design feature that graph kernel networks can train on one resolution and generalize to another resolution. The test errors on the diagonal ($s=s'=16$ and $s=s'=31$) are the smallest, which means the network has the best performance when the training grid and the test grid are the same. Interestingly, for the second row, when training on $s=31$, it is easier to general to $s'=61$ than to $s'=16$. This is because when generalizing to a larger grid, the support of the kernel becomes large which does not hurt the performance. But when generalizing to a smaller grid, part of the support of the kernel is lost, which causes the kernel to be inaccurate.

\paragraph{Expressiveness and overfitting.}
We compare the training error and test error with a different number of training pairs $N$ to see if the kernel network can learn the kernel structure even with a small amount of data. We study the expressiveness of the kernel network, examining how it overfits. We fix $r=0.10$ on the $s=s'=31$ grid and train  with $N = 10,100,1000$ whilst employing $2000$, $500$, $100$ epochs respectively. 

\begin{table}[h]

\caption{Comparing number of training pairs}
\label{table:2}
\begin{center}
\begin{tabular}{l|ll}
\multicolumn{1}{c}{\bf Training Size}  &\multicolumn{1}{c}{\bf Training Error} &\multicolumn{1}{c}{\bf Test Error} \\
\hline 
$N=10$          &$0.0089$  &$0.0931$\\
$N=100$          &$0.0183$  &$0.0478$ \\
$N=1000$          &$0.0255$  &$0.0345$ \\
\hline
\end{tabular}\\
$2000$, $500$, $100$ epochs respectively.
\end{center}

\end{table}
We see from Table \ref{table:2} that the kernel network already achieves a reasonable result when $N=10$, and the accuracy is competitive when $N=100$.
In all three cases, the test error is larger than the training error suggesting that the kernel network has enough expressiveness to overfit the training set. This overfitting is not severe as the training error will not be pushed to zero even for $N=10$, after $2000$ epochs. 

\subsection{Semi-Supervised Setting}
In the semi-supervised setting, we are only given $m$ nodes sampled from a $s\times s$ grid for each training pair, and want to evaluate on $m'$ nodes sampled from an $s'\times s'$ grid for each test pair. To be concrete, we set the number of sampled nodes $m=m'=200$. For each training pair, we sample twice $l=2$; for each test pair, we sample once $l'=1$. We train on $N=100$ pairs and test on $N'=100$ pairs. The radius for both training and testing is set to $r=r'=0.25$.

\paragraph{Generalization of resolutions on sampled grids.}
Similar to the first experiments, we train the graph kernel network with nodes sampled from the $s \times s $ resolution and test on nodes sampled from the $s'\times s'$ resolution. 
\begin{table}[h]

\caption{Generalization of resolutions on sampled Grids}
\label{table:3}
\begin{center}
\begin{tabular}{l|lll}
\multicolumn{1}{c}{\bf Resolutions}  &\multicolumn{1}{c}{\bf $s'=61$}
&\multicolumn{1}{c}{\bf $s'=121$}
&\multicolumn{1}{c}{\bf $s'=241$} \\
\hline 
$s=16$         &$0.0717$  &$0.0768$ &$0.0724$\\
$s=31$         &$0.0726$  &$0.0710$ &$0.0722$ \\
$s=61$         &$0.0687$  &$0.0728$ &$0.0723$ \\
$s=121$        &$0.0687$  &$0.0664$ &$0.0685$ \\
$s=241$        &$0.0649$  &$0.0658$ &$0.0649$ \\
\hline
\end{tabular}\\
\small{$N=100$, $m=m'=200$, $r=r'=0.25$, $l=2$}
\end{center}

\end{table}
As shown in Table \ref{table:3}, for each row, the test errors on different resolutions are about the same, which means the graph kernel network can also generalize in the semi-supervised setting. Comparing the rows, large training resolutions $s$ tend to have a smaller error. When sampled from a finer grid, there are more edges because the support of the kernel is larger on the finer grid. Still, the performance is best when $s=s'$.

\paragraph{The number of examples vs the times of sampling.}
Increasing the number of times we sample, $l$, reduces the error from the Nystr\"om approximation. By comparing different $l$ we determine which value will be sufficient.
When we sample $l$ times for each equation, we get a total of $N l$ sampled training pairs, Table~\ref{table:4}.
%

\begin{table}[h]

\caption{
Number of training pairs and sampling
}
\label{table:4}
\begin{center}
\begin{tabular}{l|llll}
\multicolumn{1}{c}{\bf } 
&\multicolumn{1}{c}{\bf $l=1$}
&\multicolumn{1}{c}{\bf $l=2$} 
&\multicolumn{1}{c}{\bf $l=4$}
&\multicolumn{1}{c}{\bf $l=8$}\\
\hline 
$N=10$ &$0.1259$ &$0.1069$ & $0.0967$ &$0.1026$\\
$N=100$ &$0.0786$ &$0.0687$ & $0.0690$ &$0.0621$\\
$N=1000$ &$0.0604$ &$0.0579$ & $0.0540$ &$0.0483$ \\
\hline
\end{tabular}\\
\small{$s=241$, $m=m'=200$, $r=r'=0.25$}
\end{center}

\end{table}
Table \ref{table:4} indicates that the larger $l$ the better, but $l=2$ already gives a reasonable performance. Moreover, the same order of sampled training pairs, $(N=100, l=8)$, and $(N=1000, l=1)$, result in a similar performance. It implies that in a low training data regime, increasing $l$ improves the performance.

\paragraph{Different number of nodes in training and testing.}
To further examine the Nystr\"om approximation, we compare different numbers of node samples $m, m'$ for both training and testing.
\begin{table}[h]

\caption{Number of nodes in the training and testing}
\label{table:5}
\begin{center}
\begin{tabular}{l|llll}
\multicolumn{1}{c}{\bf } 
&\multicolumn{1}{c}{\bf $m'=100$} 
&\multicolumn{1}{c}{\bf $m'=200$} 
&\multicolumn{1}{c}{\bf $m'=400$}
&\multicolumn{1}{c}{\bf $m'=800$}\\
\hline 
$m=100$ &$0.0871$ &$0.0716$ &$0.0662$ &$0.0609$ \\
$m=200$ &$0.0972$ &$0.0734$ &$0.0606$ &$0.0562$ \\
$m=400$ &$0.0991$ &$0.0699$ &$0.0560$ &$0.0506$\\
$m=800$ &$0.1084$ &$0.0751$ &$0.0573$ &$0.0478$ \\
\hline
\end{tabular}\\
\small{$s=121, r=r'=0.15, l=5$}
\end{center}

\end{table}

As can be seen from Table \ref{table:5}, in general the larger $m$ and $m'$ the better. For each row, fixing $m$, the larger $m'$ the better. But for each column, when fixing $m'$, increasing $m$ may not lead to better performance. This is again due to the fact that when learning on a larger grid, the kernel network learns a kernel with larger support. When evaluating on a smaller grid, the learned kernel will be truncated to have small support, leading to an increased error.
In general, $m=m'$ will be the best choice.   

\paragraph{The radius and the number of nodes.}
The computation and storage of graph networks directly scale with the number of edges. In this experiment, we want to study the trade-off between the number of nodes $m$ and the radius $r$ when fixing the number of edges.

\begin{table}[h]

\caption{The radius and the number of nodes }
\label{table:6}
\begin{center}
\begin{tabular}{l|lll}
\multicolumn{1}{c}{ }
&\multicolumn{1}{c}{\bf $m=100$}
&\multicolumn{1}{c}{\bf $m=200$}
&\multicolumn{1}{c}{\bf $m=400$}\\
\hline 
$r=0.05$  &$0.110(176)$ &$0.109(666)$ &$0.099(3354)$\\
$r=0.15$  &$0.086(512)$ &$0.070(2770)$ &$0.053(14086)$\\
$r=0.40$   &$0.064(1596)$ &$0.051(9728)$  &$0.040(55919)$ \\
$r=1.00$  &$0.059(9756)$ &$0.048(38690)$ &$-$\\
\hline
\end{tabular}\\
\small{Error (Edges), $s=121$, $l=5$, $m'=m$}
\end{center}

\end{table}

Table \ref{table:6} shows the test error with the number of edges for different $r$ and $m$. In general, the more edges, the better. For a fixed number of edges, the performance depends more on the radius $r$ than on the number of nodes $m$. In other words, the error of truncating the kernel locally is larger than the error from the Nystr\"om approximation. It would be better to use larger $r$ with smaller $m$.

\paragraph{Inner Kernel Network $\kappa$.}
To find the best network structure of $\kappa$, we compare different combinations of width and depth. We consider three cases of $\kappa$: 1. a $2-$layer feed-forward network with widths $(6, width, n^2)$, 2. a $3-$layer feed-forward network with widths $(6, width/2, width, n^2)$, and 3. a $5-$layer feed-forward network with widths $(6, width/4, width/2, width, width, n^2)$, all with ReLU activation and learning rate $1e-4$.

\begin{table}[h]
\caption{Comparing the width and depth for the inner kernel network $\kappa$}
\label{table:8}
\begin{center}
\begin{tabular}{l|lll}
\multicolumn{1}{c}{ } 
&\multicolumn{1}{c}{  depth $ = 2$}
&\multicolumn{1}{c}{  depth $ = 3$} 
&\multicolumn{1}{c}{  depth $ = 5$}\\
\hline 
width = $64$ & $0.0685$ & $0.0695$ & $0.0770$\\
width = $128$ & $0.0630$ & $0.0633$ & $0.0702$\\
width = $256$ & $0.0617$ & $0.0610$ & $0.0688$\\
width = $1024$ & $0.0641$ & $0.0591$ & $0.0608$\\
width = $4096$ & $0.2934$ & $0.0690$ & $0.0638$\\
\hline
\end{tabular}\\
\small{$s=241$, $m=200$, $r=0.25$}
\end{center}
\end{table}

As shown in Table \ref{table:8}, have wider and deeper network increase the expensiveness of the kernel network. The diagonal combinations $width=256,depth=2$, $width=1024,depth=3$, and $width=4096,depth=5$ have better test error. Notice the wide but shallow network $width=4096,depth=2$ has very bad performance. Both its training and testing error once decreased to $0.09$ around $10$th epochs, but then blow off. The $1e-4$ learning rate is probably too high for this combination. In general, depth of $3$ with with of $256$ is a good combination of Dracy Equation dataset.

\subsection{Comparison with Different Benchmarks}
In the following section, we compare the Graph Kernel Network with different benchmarks on a larger dataset of $N=1024$ training pairs computed on a $421 \times 421$ grid. The network is trained and evaluated on the same full grid. The results are presented in Table \ref{table:7}. \begin{itemize}
    \item {\bf NN} is a simple point-wise feedforward neural network. It is mesh-free, but performs badly due to lack of neighbor information.
    \item {\bf FCN} is the state of the art neural network method based on Fully Convolution Network \cite{Zabaras}. It has a dominating performance for small grids $s=61$. But fully convolution networks are mesh-dependent and therefore their error grows when moving to a larger grid.
    \item {\bf PCA+NN} is an instantiation of the methodology proposed in \cite{Kovachki}: using PCA as an autoencoder on both the input and output data and interpolating the latent spaces with a neural network. The method provably obtains mesh-independent error and can learn purely from data, however the solution can only be evaluated on the same mesh as the training data.  
    \item {\bf RBM} is the classical Reduced Basis Method (using a PCA basis), which is widely used in applications and provably obtains mesh-independent error \cite{DeVoreReducedBasis}. It has the best performance but the solutions can only be evaluated on the same mesh as the training data and one needs knowledge of the PDE to employ it.
    \item {\bf GKN} stands for our graph kernel network with $r=0.25$ and $m=300$. It enjoys competitive performance against all other methods while being able to generalize to different mesh geometries. Some figures arising from GKN are included in Appendix \ref{app:figures}.
\end{itemize}

\begin{table}[h]

\caption{Error of different methods}
\label{table:7}
\begin{center}
\begin{tabular}{l|llll}
\multicolumn{1}{c}{\bf Networks} 
&\multicolumn{1}{c}{\bf $s=85$}
&\multicolumn{1}{c}{\bf $s=141$} 
&\multicolumn{1}{c}{\bf $s=211$}
&\multicolumn{1}{c}{\bf $s=421$}\\
\hline 
NN       &$0.1716$  &$0.1716$  &$0.1716$ &$0.1716$\\
FCN       &$0.0253$  &$0.0493$  &$0.0727$ & $0.1097$\\
PCA+NN      &$0.0299$  &$0.0298$  &$0.0298$ & $0.0299$\\
RBM    &$0.0244$ &$0.0251$ &$0.0255$ &$0.0259$ \\
GKN     &$0.0346$   &$0.0332$  &$0.0342$ &$0.0369$\\
\hline 
\end{tabular}
\end{center}
\end{table}

\section{DISCUSSION AND FUTURE WORK}
We have introduced the concept of Neural Operator and instantiated it through \emph{graph kernel networks} designed to approximate
mappings between function spaces. They are constructed to be mesh-free and our numerical
experiments demonstrate that they have the desired property of being able to train and
generalize on different meshes. This is because the networks learn the mapping between infinite-dimensional function spaces, which can then be shared with approximations at
various levels of discretization. A further advantage is that data may be incorporated
on unstructured grids, using the Nystr\"om approximation.
We demonstrate that our method can achieve competitive performance with other mesh-free approaches
developed in the numerical analysis community, and beats state-of-the-art neural network approaches on large grids,
which are mesh-dependent. The methods developed in the numerical analysis community are less flexible
than the approach we introduce here, relying heavily on the variational structure of
divergence form elliptic PDEs.
Our new mesh-free method has many applications.
It has the potential to be a faster solver that learns from only sparse observations in physical space.
It is the only method that can work in the semi-supervised scenario when we only have measurements on some parts of the grid.
It is also the only method that can transfer between different geometries. For example, when computing the flow dynamic of many different airfoils, we can construct different graphs and train together.
When learning from irregular grids and querying new locations, our method does not require any interpolation, avoid subsequently interpolation error.

\paragraph{Disadvantage.}
Graph kernel network's runtime and storage scale with the number of edges $E=O(K^2)$. While other mesh-dependent methods such as PCA+NN and RBM require only $O(K)$. This is somewhat inevitable, because to learn the continuous function or the kernel, we need to capture pairwise information between every two nodes, which is $O(K^2)$; when the discretization is fixed, one just needs to capture the point-wise information, which is $O(K)$.
Therefore training and evaluating the whole grid is costly when the grid is large. On the other hand, subsampling to ameliorate cost loses information in the data, and causes errors which make our method less competitive than PCA+NN and RBM. 

\paragraph{Future works.}
To deal with the above problem, we propose that ideas such as multi-grid and fast multipole methods
\cite{gholami2016fft} may be combined with our approach to reduce complexity. In particular, a multi-grid approach will construct multi-graphs corresponding to different resolutions so that, within each graph, nodes only connect to their nearest neighbors. The number of edges then scale as $O(K)$ instead of $O(K^2)$. The error terms from Nystr\"om approximation and local truncation can be avoided.
Another direction is to extend the framework for time-dependent PDEs. Since the graph kernel network is itself an iterative solver with the time step $t$, it is natural to frame it as an RNN that each time step corresponds to a time step of the PDEs.

\subsection*{Acknowledgements}
Z. Li gratefully acknowledges the financial support from the Kortschak Scholars Program.
K. Azizzadenesheli is supported in part by Raytheon and Amazon Web Service. A. Anandkumar is supported in part by Bren endowed chair, DARPA PAIHR00111890035, LwLL grants, Raytheon, Microsoft, Google, Adobe faculty fellowships, and DE Logi grant. K. Bhattacharya, N. B. Kovachki, B. Liu and A. M. Stuart gratefully acknowledge the financial support of the Amy research Laboratory through the Cooperative Agreement Number W911NF-12-0022. Research was sponsored by the Army Research Laboratory and was accomplished under Cooperative Agreement Number W911NF-12-2-0022. 
The views and conclusions contained in this document are those of the authors and should not be interpreted as representing the official policies, either expressed or implied, of the Army Research Laboratory or the U.S. Government. The U.S. Government is authorized to reproduce and distribute reprints for Government purposes notwithstanding any copyright notation herein.

\newpage

\bibliographystyle{apalike}
\bibliography{ref.bib}

\newpage
\onecolumn
\appendix
\section{Appendix}

\subsection{Table of Notations}\label{table:notation}
\begin{table}[h]
\caption{Table of notations}
\begin{center}
\begin{tabular}{|l|l|}
\multicolumn{1}{c}{\bf Notation} 
&\multicolumn{1}{c}{\bf Meaning}\\
\hline 
{\bf PDE} &\\
$a \in \A$  & The input coefficient functions \\
$u \in \U$  & The target solution functions \\
$D \subset \R^d$  & The spatial domain for the PDE \\
$x \in D$  & Points in the the spatial domain \\
$P_K$ & $K$-point discretization of $D$\\
$\mathcal{F}: \A \to \U$  & The operator mapping the coefficients to the solutions\\
$\mu$ & A probability measure where $a_j$ sampled from.\\
\hline
{\bf Graph Kernel Networks} &\\
$\kappa : \R^{2(d+1)} \to \R^{n \times n}$  & The kernel maps $(x,y,a(x),a(y))$ to a $n \times n$ matrix\\
$\phi$  & The parameters of the kernel network $\kappa$ \\
$t = 0,\ldots,T$  & The time steps  \\
$\sigma $  & The activation function  \\
$v(x) \in \R^n$  & The neural network representation of $u(x)$ \\
$\nu$ & A Borel measure for $x$\\
$ a_{\epsilon}$  & Gaussian smoothing of $a$\\
$\nabla a_{\epsilon}$ & The gradients of $ a_{\epsilon}$\\
$W, P, Q, p, q$ & Learnable parameters of the networks.\\
\hline 
{\bf Hyperparameters} &\\
$N$  & The number of training pairs \\
$N'$  & The number of testing pairs \\
$s$  & The underlying resolution of training points \\
$s'$  & The underlying resolution of testing points \\
$K=s^2$  & Total number of nodes in the grid \\
$m$  & The number of sampled nodes per training pair \\
$m'$  & The number of sampled nodes per testing pair \\
$l$  & The times of re-sampling per training pair \\
$l'$  & The times of re-sampling per testing pair \\
$r$  & The radius of the ball of kernel integration for training \\
$r'$  & The radius of the ball of kernel integration for testing \\
\hline
\end{tabular}
\end{center}
\end{table}

\subsection{Nystr\"om Approximation}\label{app:proof}

\begin{proof}[Proof of Proposition \ref{prop:montecarlo}]
Let \(\{y_j\}_{j=1}^m\) be an i.i.d. sequence with \(y_j \sim \nu\). Define \(\kappa_y = \kappa(\cdot,y)\) for any \(y \in D\). Notice that 
by the reproducing property,
\begin{align*}
\E_{y \sim \nu} [\mathds{1}_{B(\cdot,r)}(\kappa_y \otimes \kappa_y)f] &= \E_{y \sim \nu} [\mathds{1}_{B(\cdot,r)} \langle \kappa_y, f \rangle \kappa_y] \\
&= \int_D \mathds{1}_{B(\cdot,r)} \langle \kappa_y, f \rangle \kappa_y \: \nu(dy) \\
&= \int_D \mathds{1}_{B(\cdot,r)} \kappa(\cdot,y) f(y) \: \nu(dy) \\
&= \int_{B(\cdot,r)} \kappa(\cdot,y) f(y) \: \nu(dy)
\end{align*}
hence
\[T = \E_{y \sim \nu} [\mathds{1}_{B(\cdot,r)}(\kappa_y \otimes \kappa_y)]\]
and similarly 
\[T_m = \frac{1}{m} \sum_{j=1}^m \mathds{1}(y_j \in B(\cdot,r)) (\kappa_{y_j} \otimes \kappa_{y_j}).\]
Define \(T^{(j)} := \mathds{1}(y_j \in B(\cdot,r)) (\kappa_{y_j} \otimes \kappa_{y_j})\) for any \(j \in \{1,\dots,m\}\) and \(T_y := \mathds{1}_{B(\cdot,r)}(\kappa_y \otimes \kappa_y)\) for any \(y \in D\), noting that
\(\E_{y \sim \nu} [T_y] = T\) and \(\E [T^{(j)}] = T\). Further we note that
\[\E_{u \sim \nu} \|T_y\|^2_{HS} \leq \E_{y \sim \nu} \|\kappa_y\|^4 < \infty\]
and, by Jensen's inequality,
\[\|T\|^2_{HS} \leq \E_{y \sim \nu} \|T_y\|^2_{HS} < \infty\]
hence \(T\) is Hilbert-Schmidt (as is \(T_m\) since it has finite rank).
We now compute,
\begin{align*}
\E \|T_m - T\|_{HS}^2 &=  \E \| \frac{1}{m} \sum_{j=1}^m T^{(j)} - T \|^2_{HS} \\
&= \E \| \frac{1}{m} \sum_{j=1}^m T^{(j)} \|^2_{HS} - 2 \langle \frac{1}{m} \sum_{j=1}^m \E [T^{(j)}], T \rangle_{HS} + \|T\|^2_{HS} \\
&= \E \| \frac{1}{m} \sum_{j=1}^m T^{(j)} \|^2_{HS} - \|T\|^2_{HS} \\
&= \frac{1}{m} \E_{y \sim \nu} \|T_y\|_{HS}^2 + \frac{1}{m^2} \sum_{j=1}^m \sum_{k \neq j}^m \langle \E[T^{(j)}], \E[T^{(k)}] \rangle_{HS} - \|T\|^2_{HS} \\
&= \frac{1}{m} \E_{y \sim \nu} \|T_y\|_{HS}^2 + \frac{m^2 - m}{m^2} \|T\|^2_{HS} - \|T\|^2_{HS} \\
&= \frac{1}{m} \big ( \E_{y \sim \nu} \|T_y\|_{HS}^2 -  \|T\|^2_{HS} \big ) \\
&= \frac{1}{m} \E_{y \sim \nu} \| T_y - T \|^2_{HS}.
\end{align*}
Setting \(C^2 = \E_{y \sim \nu} \| T_y - T \|^2_{HS}\), we now have
\[\E \|T_m - T\|_{HS}^2 = \frac{C^2}{m}.\]
Applying Jensen's inequality to the convex function \(x \mapsto x^2\) gives
\[\E \|T_m - T\|_{HS} \leq \frac{C}{\sqrt{m}}.\]
\end{proof}

\subsection{Figures of Table \ref{table:7}}\label{app:figures}
Figure \ref{fig:r141}, \ref{fig:r211}, and \ref{fig:r421} show graph kernel network's performance for the first $5$ testing examples with $s=141, 211, 421$ respectively. As can be seen in the figures, most of the error occurs around the singularity of the coefficient $a$ (where the two colors cross).
\begin{figure}\l
    \centering
    \includegraphics[width=16cm]{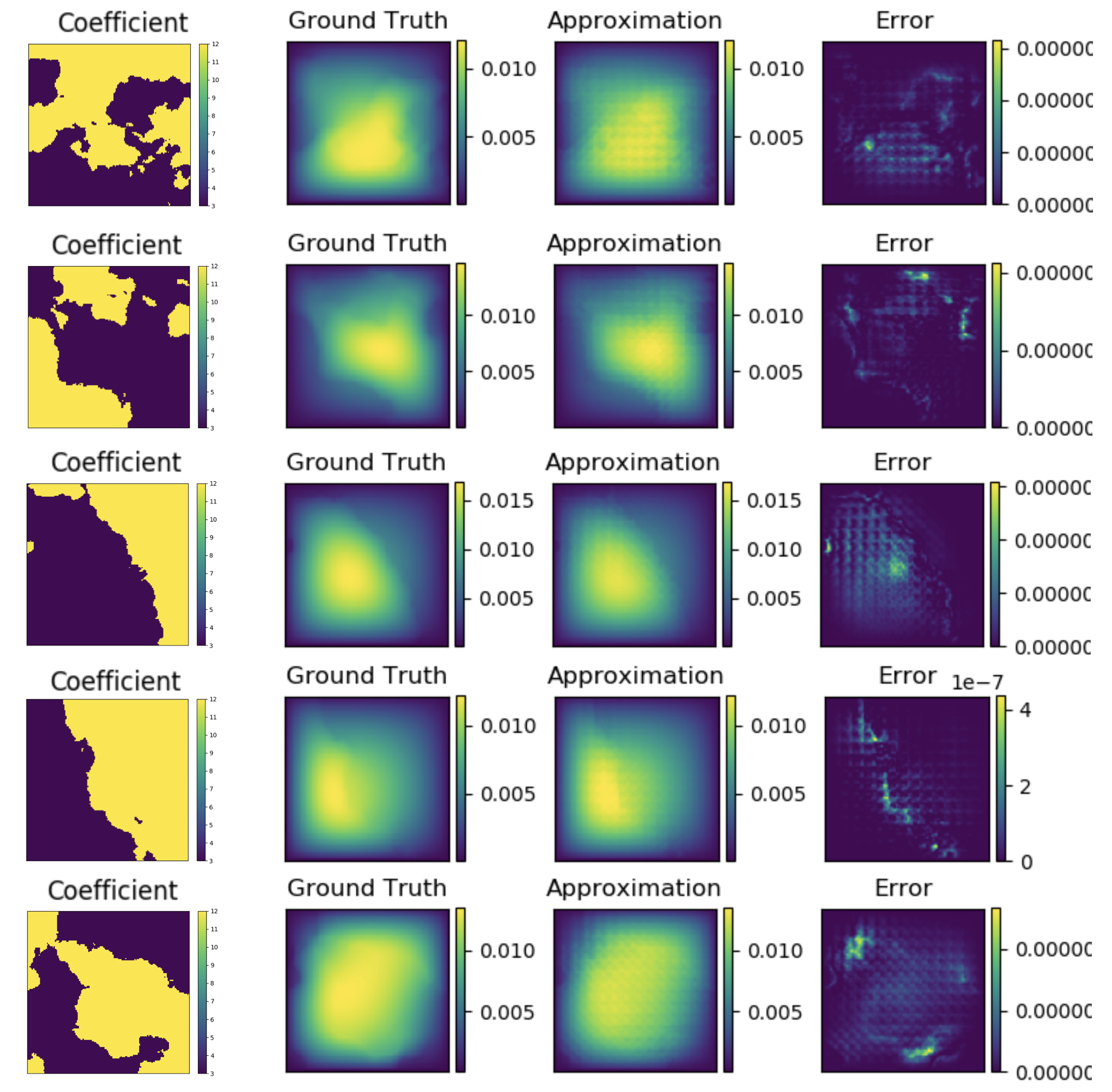}
            \caption{$s=141$}
    \label{fig:r141}
\end{figure}

\begin{figure}
    \centering
    \includegraphics[width=16cm]{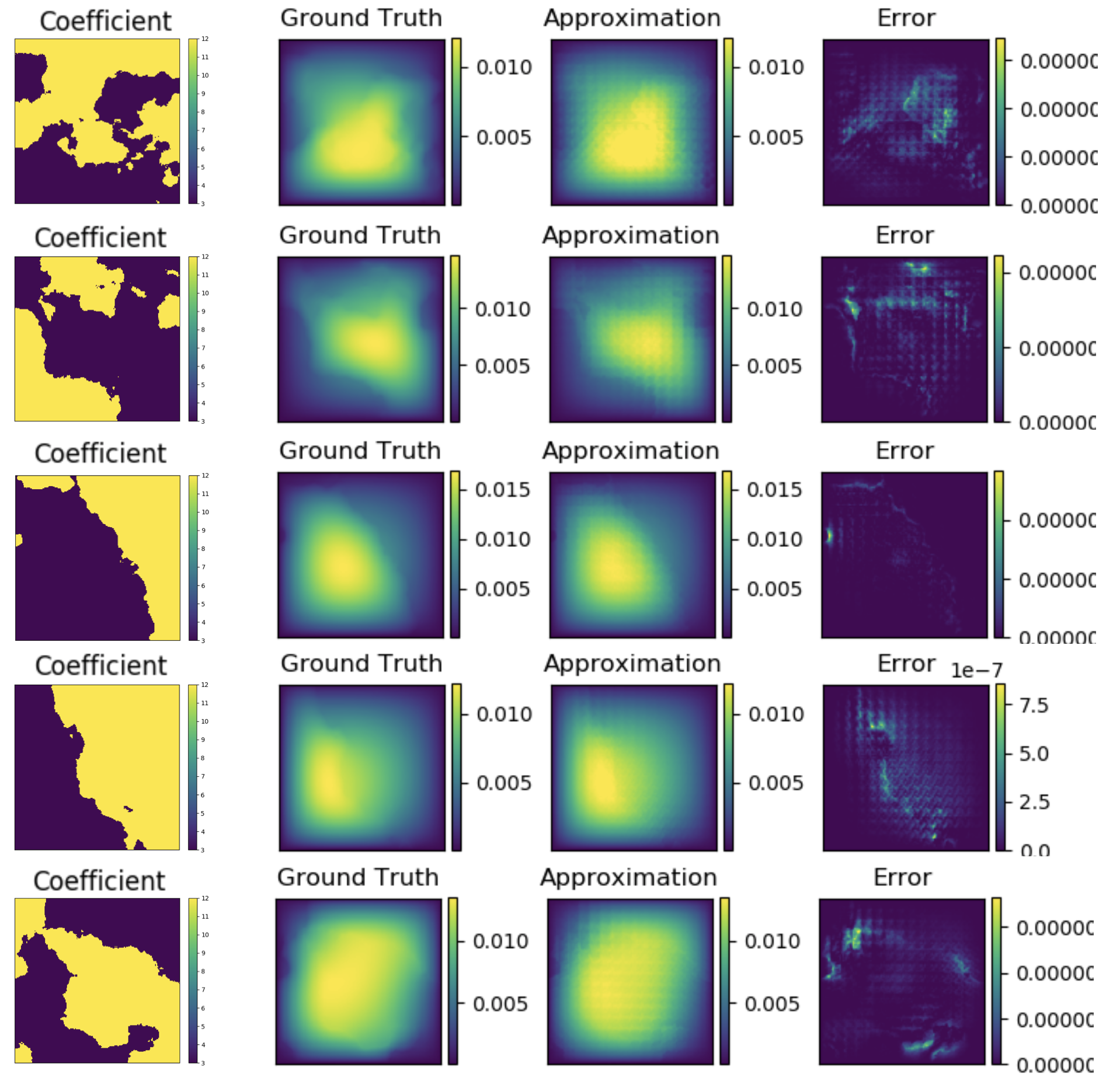}
    \caption{$s=211$}
    \label{fig:r211}
\end{figure}

\begin{figure}
    \centering
    \includegraphics[width=16cm]{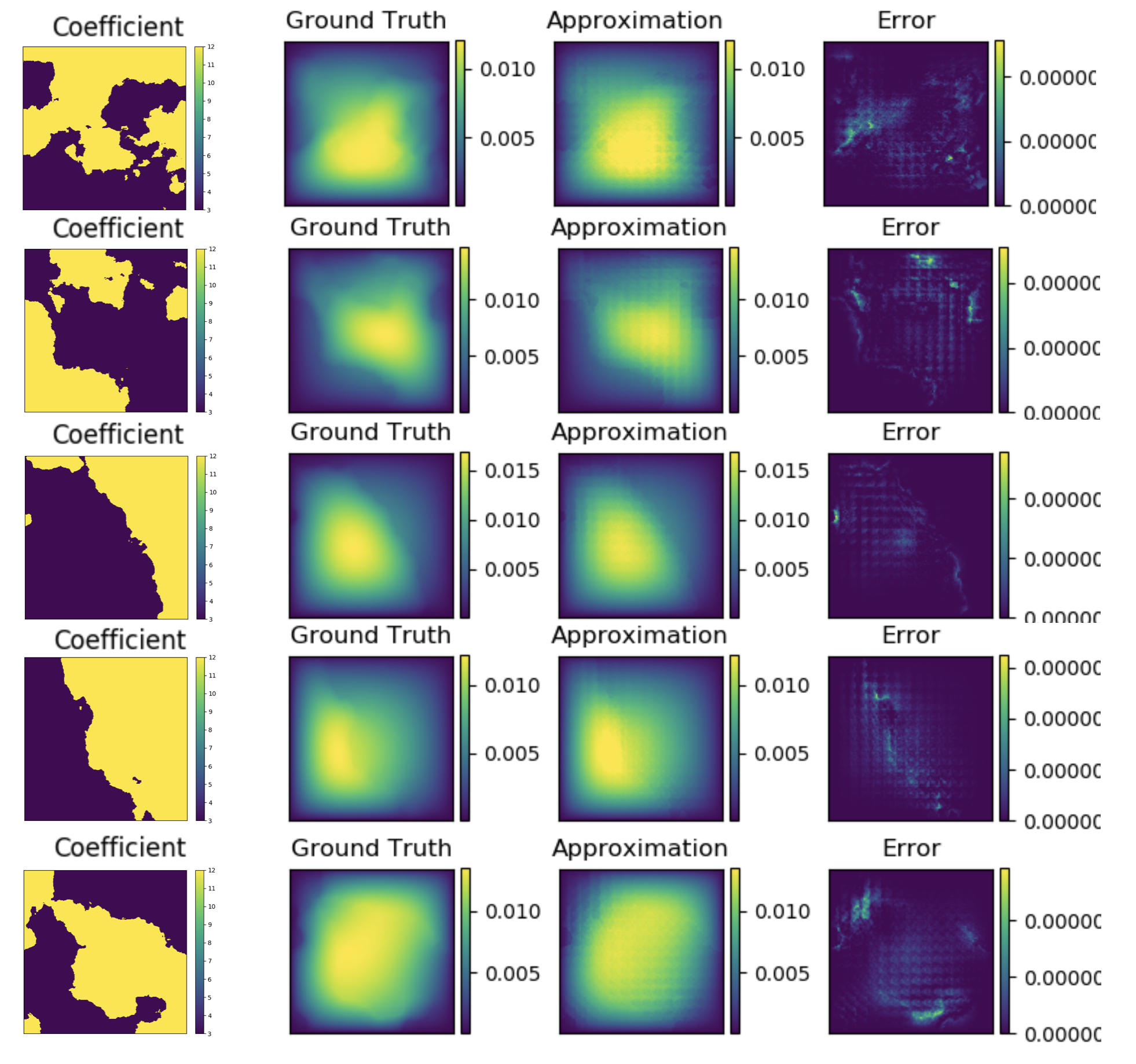}
    \caption{$s=421$}
    \label{fig:r421}
\end{figure}

\subsection{Approximating the 2D Green's function} \label{app:2D Green}
\begin{figure}
    \centering
    \includegraphics[width=15cm]{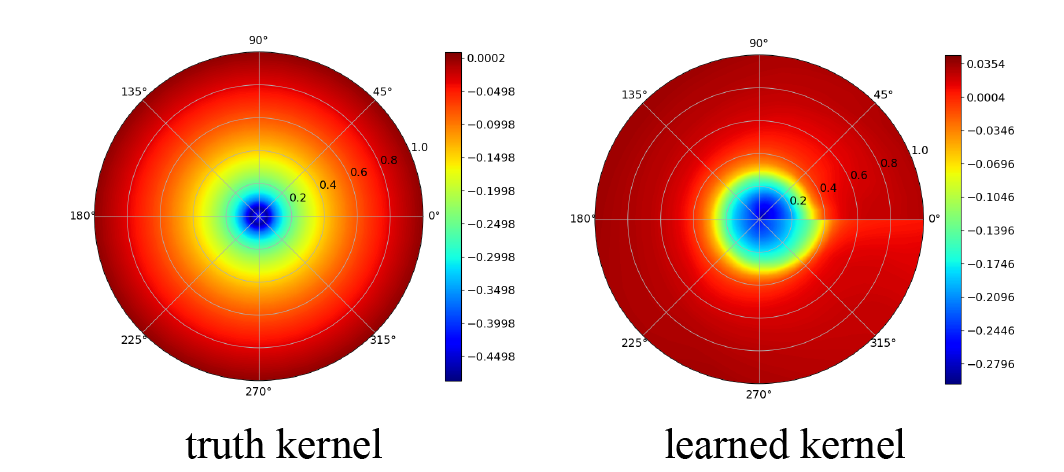}\\
    \label{fig:compareGKN}
\caption{Contour plot of the learned and truth kernel of Poisson equation on a unit disk. The contour plot shows the Green's function of \(G(0,0;\widetilde{\rho}, \widetilde{\theta})\) }
\label{fig:2D Green}
\end{figure}
In Section 3 it has been shown that the graph kernel network is capable of approximating the truth kernel of 1D Poisson equation. Here we further demonstrate the approximation power of the graph kernel network for the 2D Poisson equation by showing that the graph kernel net work is able to capture its truth Green's function. 

Although a closed-form expression for general Green's function of arbitrary 2D domain does not exist, it does exist when the partial differential equation is defined on a unit disk. Hence consider the 2D Poisson's  equation as introduced in Section 3 with $x$ defined on a unit disk \(D: x_1^2 + x_2^2 \leq 1 \). Recall the corresponding Green's function in this case can be expressed in the polar coordinate as: 
\[G_t(\rho,\theta;\widetilde{\rho}, \widetilde{\theta}) = \frac{1}{4\pi}\ln \frac{\widetilde{\rho}^2+\rho^2-2\rho\widetilde{\rho}(\cos (\widetilde{\theta})-\theta)}{\widetilde{\rho}^2\rho^2+1-2\rho\widetilde{\rho}(\cos (\widetilde{\theta})-\theta)}. \]
We first apply the graph kernel network to approximate the mapping \(f(x) \mapsto u(x)\), and then compare the learned kernel \(G\) to the truth kernel \(G_t\). Figure 7 shows contour plots of the truth and learned Green's function evaluated at \(G(0,0;\widetilde{\rho}, \widetilde{\theta})\). Again the graph kernel network is capable of learning the true Green's function, albeit the approximating error becomes larger near the singularity at zero. As a consequence, the exceptional generalization power of the graph neural network seen in Fig. 3 is explained and expected. 
\end{document}